\documentclass[11pt]{article}
\usepackage[margin=1in]{geometry}
\usepackage[T1]{fontenc}
\usepackage[utf8]{inputenc}
\usepackage{lmodern}
\usepackage{microtype}
\usepackage{amsmath,amssymb,amsthm,bm,mathtools}
\usepackage{booktabs}
\usepackage{graphicx}
\usepackage{float}
\usepackage{array}
\usepackage{multirow}
\usepackage{xurl}
\usepackage{hyperref}
\usepackage{enumitem}
\usepackage{caption}
\usepackage{subcaption}
\usepackage{longtable}
\usepackage{makecell}
\hypersetup{colorlinks=true,linkcolor=blue,citecolor=blue,urlcolor=blue}
\newtheorem{proposition}{Proposition}

\newcommand{\VR}{\mathrm{VR}}
\newcommand{\Dgm}{\mathrm{Dgm}}
\newcommand{\softmax}{\operatorname{softmax}}
\newcommand{\zscore}{\operatorname{zscore}}
\newcommand{\RMS}{\operatorname{RMS}}
\newcommand{\RMSE}{\operatorname{RMSE}}
\newcommand{\AET}{\mathrm{AET}}

\newcommand{\KH}{\mathrm{KH}}

\newcommand{\eps}{\varepsilon}

\title{Global and Local Topology-Aware Attention with Persistent Homology and Euler Biases for Time-Series Forecasting}
\author{Usef Faghihi, Amir Saki 
	\\Department of Mathematics and computer science
	\\University of Quebec Trois Rivieres
	\\usef.faghihi@uqtr.ca, amir.saki@uqtr.ca}
\date{}

\begin{document}
\maketitle

\begin{abstract}
Scientific time series often encode predictive geometric structure, including connectivity, cycles, shell-like geometry, directional changes, and nonlinear neighborhoods, that standard dot-product attention does not explicitly represent. We introduce a topology-aware attention framework that injects such structure directly into attention logits using persistent homology ($H_0$--$H_2$), anchored Euler characteristic transforms, and kernel-Hilbert channels. A validation-gated local residual further captures local topological signals, including a Zeng-style local $H_0$ component as a special case, only when held-out validation data support the correction. Capped exact Vietoris--Rips computations and smooth topological surrogates are evaluated under a no-leakage protocol with train-only calibration, validation-only selection, and test-only reporting.

We evaluate guarded topology-aware variants across three architecture families: lightweight attention/Ridge, \texttt{PatchTSTForRegression}, and \texttt{TimeSeriesTransformerForPrediction}. Experiments include synthetic benchmarks that isolate higher-order topology and real datasets covering CO$_2$, S\&P 500 return-window geometry, and NASA IMS bearing degradation. The primary audit uses matched paired comparisons across seven dataset units, three random seeds, and three chronological splits, giving 63 paired units per architecture and 189 paired units overall. Topology-aware models show positive paired effects when geometry is predictive, but the magnitude is heterogeneous across datasets and architectures. Lightweight attention/Ridge improves in 46 of 63 units (mean relative RMSE reduction $12.5\%$, paired randomization $p=7.2\times10^{-4}$); \texttt{PatchTSTForRegression} improves in 33 units and retains the baseline in 20 additional units ($23.5\%$, $p=3.5\times10^{-5}$); and \texttt{TimeSeriesTransformerForPrediction} improves in 47 units ($47.8\%$, $p<10^{-4}$). Gains are strongest on datasets with stable geometric structure and more conservative on weakly geometric tasks. The results support topology as a validation-selected, architecture-compatible inductive bias rather than as a uniformly beneficial replacement for standard temporal modeling.

\end{abstract}

\section*{Introduction}
Transformer attention models pairwise token interactions through scaled dot products \cite{vaswani2017attention}. For an input token window $X=(x_1,\ldots,x_N)$, classical attention is
\begin{equation}
A_{\mathrm{std}}(X)=\softmax\left(\frac{QK^\top}{\sqrt{d_h}}\right).
\label{eq:std_attention}
\end{equation}
This mechanism is expressive, but it has no explicit reason to preserve connected components, cycles, shell-like structure, directional Euler signatures, or nonlinear metric neighborhoods in the token cloud. Scientific time series can exhibit precisely these structures: seasonal trajectories can form loops, volatility regimes can appear as changes in return-window geometry, and bearing degradation can evolve through sparse changes in vibration features rather than smooth mean drift.

Zeng et al. \cite{zeng2021topological} showed that local persistent-homology information can be useful for time-series forecasting. Their approach is fundamentally local: topological summaries are computed from sliding windows and enter the model as a feature stream. The method presented here keeps that insight but changes the representation. First, a global module injects topology directly into the attention logits, so topology changes token interactions rather than only adding external features. Second, a separate local residual recovers the strength of local topological attention and extends it beyond scalar $H_0$ by using local $H_1$, $H_2$, and kernel-Hilbert topology. The local residual is validation-gated and does not alter the global selector.

This distinction is scientifically important. A global persistent-homology summary does not necessarily contain all local persistent-homology summaries, because persistence diagrams are many-to-one maps. Conversely, local scalar $H_0$ can be very efficient when the signal is locally coherent, but it cannot represent local loops, voids, or nonlinear Hilbert-space geometry that are invisible in one-dimensional path filtrations. The proposed framework therefore uses topology at two scales: global biases for broad geometric structure and local residuals for local coherent or higher-topological changes.

The contributions are:
\begin{enumerate}[leftmargin=*]
    \item A global topology-aware attention score that adds multiscale $H_0$, $H_1$, $H_2$, anchored Euler-transform, and RKHS/Hilbert topological channels to attention logits.
    \item A validation-gated local Hilbert/persistent-homology residual that contains a Zeng-style local $H_0$ baseline as a special case and extends it to higher local topology.
    \item A no-leakage experimental protocol with chronological splits, train-only scalers and topology calibration, validation-only selection, and test-only final reporting.
    \item Controlled synthetic and real-data experiments that clarify when local topology, global topology, and Hilbert metrics are useful.
\end{enumerate}

\section*{Related Work}
\paragraph*{Topological data analysis in machine learning.}
Persistent homology and topological data analysis summarize multiscale structure in data by tracking the births and deaths of connected components, loops, and higher-dimensional cavities across a filtration \cite{edelsbrunner2010computational,carlsson2009topology}. Because persistence diagrams are not directly Euclidean objects, a large body of work has developed stable representations that can be combined with statistical and machine-learning models, including persistence landscapes \cite{bubenik2015statistical}, persistence images \cite{adams2017persistence}, and persistence kernels \cite{reininghaus2015stable}. Directional topological summaries provide a complementary route for encoding geometry. The persistent homology transform and Euler characteristic transform scan a shape across directions and have been studied as informative shape descriptors \cite{turner2014pht,curry2022many}. Recent Euler-characteristic work is particularly relevant here: Hacquard and Lebovici show that Euler-characteristic profiles and hybrid transforms can give computationally economical topological descriptors \cite{hacquard2024euler}, Roell and Rieck introduce a differentiable Euler characteristic transform layer for graph and point-cloud learning \cite{roell2024differentiable}, von Rohrscheidt and Rieck extend ECT ideas to local graph neighborhoods \cite{vonrohrscheidt2025dissecting}, and Marsh and Beers establish stability and inference results for the ECT \cite{marsh2026stability}. These developments motivate the use of an anchored Euler-transform channel as a directional complement to persistent homology.

\paragraph*{Topological and geometric biases in attention.}
The standard Transformer computes interactions from scaled dot-product similarity \cite{vaswani2017attention}. Subsequent variants have shown that attention can be improved by adding structural information to the attention mechanism, for example through relative positional representations, segment-level recurrence with relative encodings, or explicit linear attention biases \cite{shaw2018self,dai2019transformerxl,press2022train}. Recent topological transformer work has extended this idea in several directions. Topology-Informed Graph Transformer uses topological positional embeddings and global attention for graph discrimination \cite{choi2024tigt}; the Cellular Transformer reformulates self- and cross-attention over cell complexes using incidence relations and topological positional encodings \cite{ballester2024cellular}; linear transformer topological masking parameterizes graph-topology masks while preserving linear attention scaling \cite{reid2024topologicalmasking}; and TopoFormer converts multiscale topological structure into a sequence representation for protein--ligand prediction \cite{chen2024topoformer}. The present work follows the same broad principle of structurally biasing attention, but it targets time-series forecasting and constructs additive logit-bias matrices from persistent homology, anchored Euler signatures, and kernel-Hilbert geometry. Kernel methods provide the nonlinear geometric foundation for this component: positive-definite kernels define implicit feature maps and associated Hilbert-space distances, allowing distance-based algorithms to operate in nonlinear feature spaces \cite{scholkopf2002learning,scholkopf2000distances}. In our setting, Vietoris--Rips topology is computed under such a kernel-induced Hilbert distance and injected directly into the attention matrix.

\paragraph*{Topology in time-series forecasting.}
In time-series analysis, persistent homology is commonly applied to sliding-window or delay-coordinate embeddings, where loops and related features can represent periodicity, recurrence, or regime structure \cite{perea2015sliding}. Topological machine-learning approaches have also been developed for multivariate time series by converting windows into point clouds and comparing their persistence diagrams \cite{wu2022topological}. The most direct point of comparison for this manuscript is the topological attention model of Zeng et al. \cite{zeng2021topological}, which uses local persistent-homology information within a forecasting horizon as a complementary topological signal. More recently, CrossTopoNet has incorporated persistent-homology-derived latent topological representations into a cross-attention forecasting framework \cite{lin2025crosstoponet}. These studies support the relevance of topology for forecasting, but they typically use topology as an auxiliary or separately fused representation. The present methodology instead injects multiscale $H_0$, $H_1$, $H_2$, anchored Euler, and RKHS topological information directly into the attention logits, thereby modifying the token-interaction graph itself. At the same time, the validation-gated local residual preserves the strength of local topological forecasting methods by containing a Zeng-style local $H_0$ component as a special case and extending it to local $H_1$, $H_2$, and Hilbert persistent-homology channels.

\section*{No-Leakage Protocol}
The evaluation protocol is part of the method because topology can otherwise introduce leakage through calibration or selection. For every real time-series dataset, windows are split chronologically into train, validation, and test sets. Feature scalers are fit on the training windows only. AET directions, AET thresholds, persistent-homology normalizers, RKHS distance scales, local attention projections, and local normalizers are fit only on training windows. Topology is computed only from the input window
\begin{equation}
X_t=(x_{t-L+1},\ldots,x_t),
\label{eq:input_window}
\end{equation}
never from the future target. Ridge penalties, topology modes, topological strengths, dynamic temperatures, and local/global blend weights are selected on validation data only. The test split is used once for final reporting.

This protocol also fixes the fairness of the Zeng-style reference baseline. The local $H_0$ baseline is computed from the original scaled input window rather than from patch-tokenized features, and it is not given the flattened raw input window as an extra regressor. This keeps the comparison focused on local $H_0$ topology rather than raw-window memorization.

\section*{Methodology}

The methodology is structured to proceed from the general architectural modification to increasingly specialized components. We first define the global topology-aware attention mechanism, which specifies how topological information is incorporated directly into the attention logits. We then introduce the directional and nonlinear topological bias terms that extend the basic persistent-homology formulation. Subsequently, we describe the adaptive temperature parameters used to learn the relative contribution of each topological channel. Finally, we present the localized residual correction, which accounts for fine-scale topological patterns not fully captured by the global attention bias. This organization separates the core attention modification from its specialized extensions and clarifies the role of the global and local components within the overall forecasting framework.
\subsection*{Global topology-aware attention mechanism}
Classical self-attention computes token interactions through scaled dot products, as in Eq.~\eqref{eq:std_attention}. This mechanism is flexible, but it has no explicit reason to preserve connected components, cyclic structure, shell-like geometry, directional Euler signatures, or nonlinear metric neighborhoods. We therefore augment the classical attention score with additive topological bias matrices:
\begin{equation}
A_{\mathrm{global}}(X)=
\softmax\left(
\frac{QK^\top}{\sqrt{d_h}}
+\sum_{k=0}^{2}\lambda_k(X)B^{H_k}(X)
+\lambda_E(X)B^{\AET}(X)
+\sum_{k=0}^{2}\lambda_{K,k}(X)B^{\KH_k}(X)
\right).
\label{eq:global_attention}
\end{equation}
Here $B^{H_0}$ is a multiscale connectivity bias, $B^{H_1}$ is a loop or cyclic bias, $B^{H_2}$ is a shell or void bias, $B^{\AET}$ is an anchored Euler-transform directional bias, and $B^{\KH_k}$ denotes persistent homology computed under a kernel-Hilbert distance. The strengths $\lambda_k(X)$, $\lambda_E(X)$, and $\lambda_{K,k}(X)$ are selected from the validation split or learned through the adaptive temperature mechanism described below. Thus, topology modifies the token-interaction graph directly at the attention-logit level rather than being appended only as an auxiliary feature stream.

For a finite token cloud $X=\{x_1,\ldots,x_N\}$ with distance $d$, the Vietoris--Rips complex at radius $\eps$ is
\begin{equation}
\VR_\eps(X,d)=
\left\{
\sigma\subseteq X:\max_{x_i,x_j\in\sigma}d(x_i,x_j)\leq \eps
\right\}.
\label{eq:rips}
\end{equation}
The lifetime summary used for a homology dimension $k$ is
\begin{equation}
S_k(X)=
\log\left(
1+
\sum_{(b_\ell,d_\ell)\in\Dgm_k(X)}
\frac{d_\ell-b_\ell}{\operatorname{median}\{d(x_i,x_j):i<j\}+\epsilon}
\right),
\qquad k\in\{0,1,2\}.
\label{eq:lifetime_summary}
\end{equation}
When the exact backend is available, capped Vietoris--Rips persistence is computed from the distance matrix using GUDHI-style Rips complexes. To keep the implementation executable without specialized dependencies and to avoid cubic blow-up on long windows, smooth distance-based surrogates are also implemented. The exact and surrogate versions are treated as candidate feature generators under the same validation-only selection protocol.

The smooth $H_0$ bias is a weighted multiscale radial affinity,
\begin{equation}
B^{H_0}_{ij}=\zscore\left[
0.50\exp\left(-\frac{d_{ij}^2}{2(0.5\sigma)^2}\right)
+0.35\exp\left(-\frac{d_{ij}^2}{2\sigma^2}\right)
+0.15\exp\left(-\frac{d_{ij}^2}{2(2\sigma)^2}\right)
\right],
\label{eq:h0_bias}
\end{equation}
where $\sigma$ is the median nonzero pairwise distance. The three scales form a compact multiscale approximation: the smallest scale emphasizes local connectedness, the middle scale captures the dominant neighborhood radius, and the largest scale stabilizes the bias when neighborhoods are sparse.

The smooth $H_1$ bias uses two-hop cycle-closing evidence. Define
\begin{equation}
A^{(\eps)}_{ij}=\sigma_{\mathrm{logistic}}\left(\frac{\eps-d_{ij}}{\tau_\eps}\right),
\label{eq:soft_adj}
\end{equation}
and
\begin{equation}
C^{(\eps)}_{ij}=\left[\frac{1}{N-2}\sum_k A^{(\eps)}_{ik}A^{(\eps)}_{kj}\right]\left(1-A^{(\eps)}_{ij}\right).
\label{eq:cycle_closing}
\end{equation}
This score is large when $i$ and $j$ share two-hop paths but are not directly connected, which is a local signature of cycle closure. The implemented loop bias averages across the scales
\begin{equation}
\mathcal E=\{0.70\sigma,\sigma,1.40\sigma\},
\qquad
B^{H_1}_{ij}=\zscore\left[\frac{1}{|\mathcal E|}\sum_{\eps\in\mathcal E}C^{(\eps)}_{ij}\right].
\label{eq:h1_bias}
\end{equation}
The $H_2$ term is a shell proxy combining centered radii and local sparsity:
\begin{equation}
B^{H_2}_{ij}=\zscore\left[
\exp\left(-\frac{(r_i-r_j)^2}{2s_r^2}\right)\frac{\rho_i+\rho_j}{2}
\right],
\label{eq:h2_bias}
\end{equation}
where $r_i$ is the centered token radius, $s_r$ is a robust radius scale, and $\rho_i$ is a local sparsity score. This proxy is not claimed to equal exact $H_2$ persistence; it is a smooth, low-cost bias that is validated against exact summaries when available.

\subsection*{Computational complexity}
The complexity of the proposed topology channels is reported per input window with $N$ tokens and token dimension $p$. The base attention score requires $O(N^2d_h)$ arithmetic and $O(N^2)$ memory for the attention matrix. The topology-aware implementation adds either capped exact persistence or smooth surrogate biases before the softmax.

For exact Vietoris--Rips persistence, the pairwise distance matrix costs $O(N^2p)$ time and $O(N^2)$ space. Exact persistence is then computed only on a capped subset of size $c\leq 28$ with a maximum-edge filter. If homology is computed through dimension $q=2$, the Rips construction may require simplices through dimension $q+1=3$, giving
\[
S_c=\sum_{r=1}^{q+2}\binom{c}{r}=O(c^4)
\]
simplices in the worst case. Standard boundary-matrix reduction is $O(S_c^3)$ time and $O(S_c^2)$ space. Without the cap, this worst-case bound becomes combinatorial in $N$ and is the main reason that exact persistence is not used indiscriminately on long windows. With fixed $c$, the exact persistence cost is bounded per window, but it remains more expensive than the smooth biases.

The smooth $H_0$ and $H_2$ biases require the dense distance matrix and elementwise transformations, giving $O(N^2p+N^2)$ time and $O(N^2)$ space. The smooth $H_1$ cycle-closing surrogate uses dense two-hop products across a small fixed set of radii; with the three radii used here, the naive complexity is $O(N^3)$ time and $O(N^2)$ space. The AET channel with $R$ directions and $Q$ thresholds costs $O(RNp+RQN^2)$ time and $O(N^2+RQN)$ space if memberships are materialized, or $O(N^2)$ auxiliary space if thresholds are streamed. The RKHS channel adds $O(N^2p)$ time to form kernel-induced distances and then uses either the same capped exact persistence path or the same smooth-bias path. Thus, except for the capped exact persistence and the $H_1$ surrogate, the implemented topology channels remain quadratic in the token count, matching the memory order of ordinary softmax attention.

\subsection*{Directed and nonlinear topological biases}
Standard persistent homology is stable and geometrically informative, but it can be insensitive to directional redistributions that preserve Betti numbers. This limitation is important in time series such as degradation trajectories, where the direction of structural change can carry predictive information. The Anchored Euler-transform bias (AET) complements persistent homology by encoding directional changes in sublevel-set Euler structure. The anchored Euler term scans train-only directions and thresholds. For directions $v_r$, thresholds $a_{rq}$, and temperature $\tau$, define
\begin{equation}
m_{irq}=\sigma\left(\frac{a_{rq}-\langle x_i,v_r\rangle}{\tau}\right).
\label{eq:aet_membership}
\end{equation}
The differentiable graph-Euler statistic is
\begin{equation}
\widetilde\chi_{rq}(X)=\sum_i m_{irq}-\sum_{i<j}a_{ij}m_{irq}m_{jrq}.
\label{eq:graph_euler}
\end{equation}
The token-level Euler contribution is
\begin{equation}
c_{irq}=m_{irq}\left(1-\sum_j a_{ij}m_{jrq}\right),
\label{eq:aet_contrib}
\end{equation}
and the AET bias is
\begin{equation}
B^{\AET}_{ij}(X)=\frac{1}{RQ}\sum_{r,q}c_{irq}c_{jrq}.
\label{eq:aet_bias}
\end{equation}

Euclidean persistent homology may also miss nonlinear similarities. To address this limitation, the Hilbert-space channel replaces Euclidean distances with a kernel-induced metric. The RKHS channel uses the positive-definite Gaussian kernel
\begin{equation}
	k_\ell(x_i,x_j)
	=
	\exp\left(
	-\frac{\|x_i-x_j\|^2}{2\ell^2}
	\right).
	\label{eq:rkhs_gaussian_kernel}
\end{equation}
The corresponding implicit feature map $\phi$ satisfies
\begin{equation}
	k(x_i,x_j)
	=
	\langle \phi(x_i),\phi(x_j)\rangle_{\mathcal H}.
	\label{eq:rkhs_kernel_inner}
\end{equation}
The induced Hilbert distance can then be computed without explicitly constructing $\phi$:
\begin{equation}
	d_{\mathcal H}(x_i,x_j)^2
	=
	k(x_i,x_i)+k(x_j,x_j)-2k(x_i,x_j).
	\label{eq:rkhs_hilbert_distance}
\end{equation}
Vietoris--Rips persistence is computed using $d_{\mathcal H}$ in place of the Euclidean distance, yielding topology in the associated RKHS. The bandwidth $\ell$ controls the scale of the geometry: small $\ell$ emphasizes local neighborhoods, while large $\ell$ captures more global relationships. Hilbert-space topology therefore provides a nonlinear bridge between local and global geometry, although exact recovery of local window topology still requires the explicit local cover described later.

\subsection*{Adaptive topological temperatures}
The predictive value of a given topology channel can vary across datasets and regimes. For example, $H_0$ connectivity may be useful in locally coherent financial windows, while $H_1$ loops or RKHS geometry may be more informative in seasonal or degradation trajectories. Fixed topological coefficients are therefore structurally restrictive. The dynamic variant learns nonnegative effective temperatures:
\begin{equation}
S_m(X)=\frac{QK^\top}{\sqrt{d_h}}+
\sum_{c\in\mathcal C_m}\eta_c B^c(X),
\qquad
\eta_c=\operatorname{softplus}(\alpha_c),
\label{eq:dynamic_eta}
\end{equation}
where $\mathcal C_m$ denotes the channels included in topology mode $m$. The parameters $\alpha_c$ are trained on the training split, early-stopped on validation data, and then used as features for the same validation-selected forecasting head. This lets the model decide how much attention weight to route through each topological channel rather than forcing a fixed strength for all datasets.

\subsection*{Local residual correction and preservation principle}
Global persistence diagrams are many-to-one maps and do not inherently summarize every local persistent-homology variation. There can be two input windows $X$ and $X'$ with identical global topology,
\begin{equation}
D_{\mathrm{global}}(X)=D_{\mathrm{global}}(X'),
\label{eq:global_many_to_one_equal}
\end{equation}
but different local diagram collections,
\begin{equation}
\{D_m^0(X)\}_{m=1}^{M}\neq
\{D_m^0(X')\}_{m=1}^{M}.
\label{eq:local_diagrams_differ}
\end{equation}
Consequently, no predictor depending only on $D_{\mathrm{global}}(X)$ can reproduce every possible local persistent-homology representation. A model that preserves the global topology-aware mechanism while capturing localized coherent signals must therefore retain an explicit local cover.

At the attention-logit level, the clean extension is
\begin{equation}
A_{\mathrm{aug}}(X)=
\softmax\left(S_{\mathrm{global}}(X)+\alpha_{\mathrm{loc}}(X)S_{\mathrm{loc}}(X)\right),
\label{eq:augmented_attention}
\end{equation}
where
\begin{equation}
S_{\mathrm{global}}(X)=
\frac{QK^\top}{\sqrt{d_h}}
+\sum_{k=0}^{2}\lambda_k(X)B^{H_k}(X)
+\lambda_E(X)B^{\AET}(X)
+\sum_{k=0}^{2}\lambda_{K,k}(X)B^{\KH_k}(X),
\label{eq:sglobal}
\end{equation}
and
\begin{equation}
S_{\mathrm{loc}}(X)=
\sum_{m=1}^{M}
\sum_{c\in\mathcal C_{\mathrm{loc}}}
\gamma_{m,c}(X)\,
M_m B_m^{c}(X_m)M_m^\top.
\label{eq:sloc}
\end{equation}
Here $X_m$ is the $m$-th local subwindow, $M_m$ is the binary mask that embeds the local bias back into the full token window, and
\begin{equation}
\mathcal C_{\mathrm{loc}}=
\{H_0,H_1,H_2,\KH_0,\KH_1,\KH_2\}.
\label{eq:local_channels}
\end{equation}
The local scalar $H_0$ channels recover Zeng-style local information, while the additional channels add local loop, void, and Hilbert-space geometry.

This formulation preserves the current global method exactly. If
\begin{equation}
\alpha_{\mathrm{loc}}(X)=0,
\label{eq:alpha_zero_attention}
\end{equation}
then
\begin{equation}
A_{\mathrm{aug}}(X)=A_{\mathrm{global}}(X).
\label{eq:augmented_preserves_global_attention}
\end{equation}
Therefore, if $\mathcal F_{\mathrm{global}}$ is the predictor class induced by the global module and $\mathcal F_{\mathrm{aug}}$ is the predictor class after adding the local residual, then
\begin{equation}
\mathcal F_{\mathrm{global}}\subseteq \mathcal F_{\mathrm{aug}}.
\label{eq:global_containment}
\end{equation}
This containment protects representational capacity but does not by itself guarantee lower test error. The validation gate decides whether the extra local module is supported by held-out validation data.

Local PH contrasts provide the gate for the local score. Let $C_m^{(c)}(X)$ summarize the normalized change in channel $c$ around local window $m$. A contrast-gated version is
\begin{equation}
\gamma_m(X)=
\softmax_m\left(
u^\top
\left[
C_m^{(H_0)},C_m^{(H_1)},C_m^{(H_2)},
C_m^{(\KH_0)},C_m^{(\KH_1)},C_m^{(\KH_2)}
\right]
\right).
\label{eq:local_gamma}
\end{equation}
Equivalently, $\gamma_{m,c}$ in Eq.~\eqref{eq:sloc} may be channel-specific. The purpose is to give more local attention weight to subwindows where topology changes sharply. The implementation uses the same idea in a representation-level form: the normalized local contrast matrix is added to local attention logits, and contrast statistics are concatenated to the local Hilbert/PH representation. The global method remains preserved because the validation gate can set $\alpha_{\mathrm{loc}}=0$.

The local Hilbert/persistent-homology residual generalizes prior local topological approaches. Zeng-style methods use local $H_0$ connectivity summaries computed from sliding windows. However, scalar local $H_0$ summaries cannot represent local loops, voids, or nonlinear Hilbert-space geometry. The proposed residual computes richer local topological features over a cover of the input window, including scalar $H_0$ diagrams from both the original and negated one-dimensional path filtrations, higher-dimensional persistent-homology diagrams, and RKHS-based diagrams. These local features are vectorized, mapped through a local attention projection, and added to the global prediction only when validated by held-out performance.

Let $\mathcal U(X)=\{U_{m,s}\}$ be a cover of the input window by local subwindows and selected larger-scale subwindows. For each local element, the module computes
\begin{equation}
	D_{m,s}^{0,+},\quad
	D_{m,s}^{0,-},\quad
	D_{m,s}^{1},\quad
	D_{m,s}^{2},\quad
	D_{m,s}^{\KH_0},\quad
	D_{m,s}^{\KH_1},\quad
	D_{m,s}^{\KH_2}.
	\label{eq:local_diagrams}
\end{equation}
Here $D_{m,s}^{0,+}$ and $D_{m,s}^{0,-}$ are scalar local $H_0$ diagrams obtained from the original and negated one-dimensional path filtrations, respectively. The diagrams $D_{m,s}^{1}$ and $D_{m,s}^{2}$ are computed from pairwise Euclidean distances, while $D_{m,s}^{\KH_0}$, $D_{m,s}^{\KH_1}$, and $D_{m,s}^{\KH_2}$ are computed from RKHS-induced distances.

Each diagram is vectorized using finite-lifetime statistics, including the largest lifetimes, total persistence, mean lifetime, standard deviation, maximum lifetime, and number of finite bars. Denote this vectorization by $\varphi$. The local feature vector is then
\begin{equation}
	\Phi_{m,s}(X)=
	\left[
	\varphi(D_{m,s}^{0,+}),
	\varphi(D_{m,s}^{0,-}),
	\varphi(D_{m,s}^{1}),
	\varphi(D_{m,s}^{2}),
	\varphi(D_{m,s}^{\KH_0}),
	\varphi(D_{m,s}^{\KH_1}),
	\varphi(D_{m,s}^{\KH_2}),
	\operatorname{stats}(U_{m,s})
	\right].
	\label{eq:local_phi}
\end{equation}
A train-only attention projection maps these local vectors to a local representation
\begin{equation}
	R_{\mathrm{loc}}(X)
	=
	\operatorname{Attn}\left(\Phi_{1}(X),\ldots,\Phi_M(X)\right).
	\label{eq:local_attention_rep}
\end{equation}

The module also uses a normalized local persistent-homology contrast. For two local subwindows $U_a$ and $U_b$ in channel $k$, define
\begin{equation}
	C_{ab}^{(k)}(X)=
	\frac{
		\left\|\varphi(D_a^{(k)})-\varphi(D_b^{(k)})\right\|_{\RMS}
	}{
		\sqrt{
			\tfrac{1}{2}
			\left(
			\|\varphi(D_a^{(k)})\|_{\RMS}^2
			+
			\|\varphi(D_b^{(k)})\|_{\RMS}^2
			\right)
		}
		+\epsilon
	}.
	\label{eq:local_contrast}
\end{equation}
This contrast is observational and input-window-only; it is not an interventional causal estimand. Its role is to make the residual sensitive to changes in local topology across the cover.

The final prediction is obtained by a validation-gated residual correction:
\begin{equation}
	\widehat y_{\mathrm{final}}(X)
	=
	\widehat y_{\mathrm{global}}(X)
	+
	\alpha_{\mathrm{loc}}\widehat r_{\mathrm{local}}(X),
	\qquad
	\widehat r_{\mathrm{local}}(X)
	=
	\widehat y_{\mathrm{local}}(X)
	-
	\widehat y_{\mathrm{global}}(X).
	\label{eq:guarded_residual}
\end{equation}
Equivalently,
\begin{equation}
	\widehat y_{\mathrm{final}}(X)
	=
	(1-\alpha_{\mathrm{loc}})\widehat y_{\mathrm{global}}(X)
	+
	\alpha_{\mathrm{loc}}\widehat y_{\mathrm{local}}(X).
	\label{eq:guarded_blend}
\end{equation}
The coefficient $\alpha_{\mathrm{loc}}$ is selected only on the validation set. Let $\alpha^\star$ be the validation-optimal blend over a fixed grid. The reported model uses
\begin{equation}
	\alpha_{\mathrm{loc}}=
	\begin{cases}
		\alpha^\star,
		&
		\RMSE_{\mathrm{val}}(\mathrm{global+local})
		<
		\RMSE_{\mathrm{val}}(\mathrm{global})
		-
		\delta_{\mathrm{loc}}\max\{1,\RMSE_{\mathrm{val}}(\mathrm{global})\},
		\\
		0,
		&
		\text{otherwise.}
	\end{cases}
	\label{eq:residual_guard}
\end{equation}
Therefore, when the local residual is not validated, $\alpha_{\mathrm{loc}}=0$ and the selected global prediction is preserved exactly.

\begin{proposition}[Containment of the local $H_0$ baseline]
Let $\mathcal F_Z$ denote the class of predictors generated by a Zeng-style local $H_0$ representation using local scalar path diagrams from $X$ and $-X$. Let $\mathcal F_L$ denote the class generated by the local Hilbert/persistent-homology representation in Eq.~\eqref{eq:local_phi}. Then
\begin{equation}
\mathcal F_Z\subseteq \mathcal F_L.
\label{eq:containment}
\end{equation}
\end{proposition}

\begin{proof}
The representation $\Phi_{m,s}(X)$ contains $\varphi(D_{m,s}^{0,+})$ and $\varphi(D_{m,s}^{0,-})$, the two diagram summaries used by the controlled local $H_0$ baseline. Set the weights on all non-$H_0$ coordinates, including $H_1$, $H_2$, RKHS coordinates, and auxiliary statistics, to zero. Choose the remaining attention and forecasting weights equal to those of the local $H_0$ baseline. The resulting predictor is exactly the baseline predictor. Therefore every predictor in $\mathcal F_Z$ can be represented in $\mathcal F_L$.
\end{proof}

Containment does not imply that the local module must beat local $H_0$ on every finite sample. It proves representational capacity. Strict improvement is possible when there exist $X$ and $X'$ such that the local $H_0$ representation is identical but local $H_1$, $H_2$, or RKHS topology differs and the target depends on that difference. This is precisely what the higher-topology stress test is designed to evaluate.

The final design gives three guarantees. First, the global method is preserved because Eq.~\eqref{eq:residual_guard} can set $\alpha_{\mathrm{loc}}=0$, exactly recovering $\widehat y_{\mathrm{global}}$. Second, the Zeng-style local $H_0$ baseline is contained because the local module can zero out all non-Zeng channels and use only local $H_0$ diagrams from $X$ and $-X$. Third, improvement over local $H_0$ is possible on data-generating processes where the target depends on local $H_1$, $H_2$, or RKHS/Hilbert topology that is not measurable from scalar local $H_0$ diagrams. The third point is conditional, not universal: validation-only selection determines whether this extra information is useful for a given dataset.

\section*{Experimental Setup}
\subsection*{Implementation details}
All experiments follow the no-leakage protocol defined above. The dataset construction is shared across the architecture families so that the same seven non-degenerate dataset-level comparisons are used for every completed model run.

The executed package contains three architecture groups. First, the original topology-aware attention forecaster is a lightweight attention summary followed by a Ridge regression head. It is evaluated with per-timestep tokens and non-overlapping patch tokens. Its global topology-aware variant adds the manuscript channels in Eq.~\eqref{eq:global_attention} to the attention logits, namely $H_0$, $H_1$, $H_2$, AET, and RKHS/Hilbert $KH_0$, $KH_1$, and $KH_2$. Its local correction follows the guarded residual in Eq.~\eqref{eq:guarded_residual}--Eq.~\eqref{eq:residual_guard}. Ridge penalties are selected from $\{0.001,0.01,0.1,1,10,50,100\}$. Static topological strengths and blend weights are selected on the validation split. Dynamic configurations learn nonnegative temperatures $\eta_c=\operatorname{softplus}(\alpha_c)$ for 16 epochs with learning rate $0.03$, weight decay $10^{-4}$, and early stopping with patience 5.

Second, the architecture-level code instantiates \texttt{PatchTSTForRegression} in two forms: an ordinary baseline and a topology-aware version. In the topology-aware form, the same input-window-only topological channels in Eq.~\eqref{eq:global_attention} are calibrated from the training split and added as attention-logit biases over the patch-token sequence. The topology-aware PatchTST run also retains the local Hilbert/PH residual as a validation-gated correction, so that the reported topology-aware PatchTST column follows the same global-bias and local-residual design as the manuscript.

Third, the architecture-level code instantiates \texttt{TimeSeriesTransformerForPrediction} in two forms: an ordinary baseline and a topology-aware version. The topology-aware version applies the same topological bias construction to the encoder self-attention logits over the past-value window, with the same nonnegative channel-temperature parameterization and the same validation-gated residual convention. Thus, each Transformer architecture is evaluated once without topology and once with the manuscript topological methodology.

Exact persistent homology is used when computationally feasible, with a cap of 28 points and a maximum-edge quantile of 0.60; otherwise, the smooth surrogate biases in Eq.~\eqref{eq:h0_bias}--Eq.~\eqref{eq:h2_bias} are used under the same validation-only selection protocol. Local subwindows use length 8 and stride 4. Each local diagram is summarized by finite-lifetime statistics including the top 4 finite lifetimes per channel, and local features are projected through a 16-dimensional train-only attention layer.

\subsection*{Datasets}
The dataset descriptions are given once here and apply to all three architecture groups. Performance is assessed on seven non-degenerate dataset-level comparisons. Three controlled synthetic benchmarks isolate higher-dimensional topology, and four real-data comparisons test seasonal, financial, and bearing-degradation settings. A third IMS folder is audited but excluded because its target sanity checks indicate a near-degenerate target.
\begin{enumerate}[leftmargin=*]
    \item \textbf{Higher-topology stress test.} Each input is a two-dimensional point cloud with matched coordinate marginals. In the positive regime,
    \begin{equation}
    X_i=(\cos\theta_i,\sin\theta_i)+\epsilon_i,
    \label{eq:stress_loop}
    \end{equation}
    while in the negative regime the second coordinate is independently permuted,
    \begin{equation}
    \widetilde X_i=(\cos\theta_i,\sin\theta_{\pi(i)})+\epsilon_i.
    \label{eq:stress_scramble}
    \end{equation}
    Both regimes are randomly rotated and token-permuted. Scalar coordinate marginals remain similar, but the positive regime preserves a global loop ($H_1$ structure) that is destroyed in the negative regime.
    \item \textbf{Cyclic $H_1$ benchmark.} Multivariate windows are generated from sinusoidal trajectories with randomized amplitudes, phases, and frequencies. The target is the next value of the primary sinusoidal component.
    \item \textbf{Shell/void $H_2$ benchmark.} Windows are sampled either from a noisy spherical shell or a filled ball. The target is a binary indicator of the shell regime.
    \item \textbf{Monthly CO$_2$.} The bundled monthly atmospheric CO$_2$ record is converted to windows containing the CO$_2$ value and seasonal sine/cosine coordinates. The target is the next monthly CO$_2$ value.
    \item \textbf{S\&P 500 volatility.} The FRED S\&P 500 index is converted to log returns
    \begin{equation}
    r_t=\log(P_t)-\log(P_{t-1}),
    \label{eq:sp500_returns}
    \end{equation}
    and the target is 5-day-ahead annualized realized volatility,
    \begin{equation}
    Y_t=\sqrt{\frac{1}{5}\sum_{j=1}^{5}r_{t+j}^2\times 252}.
    \label{eq:sp500_target}
    \end{equation}
    Topology is computed exclusively from past return-window features.
    \item \textbf{NASA IMS bearing sets 1 and 2.} The IMS bearing data are analyzed per test because channel layouts differ. Set 1 uses bearing-channel groups $(1,2)$, $(3,4)$, $(5,6)$, and $(7,8)$; set 2 uses one channel per bearing. For each bearing $b$, the health indicator is
    \begin{equation}
    \mathrm{HI}_b(t)=0.55z_{\mathrm{RMS},b}(t)+0.25z_{\mathrm{STD},b}(t)+0.20z_{\mathrm{KURT},b}(t),
    \label{eq:ims_hi}
    \end{equation}
    followed by median smoothing, rolling averaging, nonnegative trend extraction, and next-snapshot prediction. The third IMS folder is skipped because the target sanity check indicates an extremely low-variance target, making RMSE comparisons misleading.
\end{enumerate}

The architecture-level dataset-loading logs report the following input-window tensors: higher-topology stress $[300,32,2]$, cyclic $H_1$ $[260,24,3]$, shell/void $H_2$ $[260,24,3]$, monthly CO$_2$ $[260,30,3]$, S\&P 500 volatility $[1192,40,6]$, IMS set 1 $[1114,24,45]$, and IMS set 2 $[960,24,21]$. The same seven datasets are therefore available to the lightweight, PatchTST, and TimeSeriesTransformer architecture sections.

\subsection*{Baselines and metrics}
For the lightweight attention/Ridge architecture, the completed baselines are classical attention without topological biases and a controlled Zeng-style local $H_0$ baseline. The Zeng-style row is an adaptation for controlled comparison, not a claim to reproduce every training detail of the original paper. It receives only scalar local $H_0$ summaries, including summaries from the original scaled input windows and their negations. It does not receive local $H_1$, local $H_2$, AET, RKHS distances, flattened raw windows, or patch-expanded features. This isolates the effect of local $H_0$ topology.

For \texttt{PatchTSTForRegression} and \texttt{TimeSeriesTransformerForPrediction}, the comparison is the ordinary architecture versus the same architecture with the topological logit-bias and guarded-residual mechanisms described above. RMSE is the primary metric and MAE values are stored in the corresponding result folders. The architecture-level result folders contain completed aggregate CSVs for all seven non-degenerate datasets.

The statistical audit is applied after all validation decisions are fixed. For each architecture, the primary inferential analysis uses $7$ datasets, $3$ random seeds, and $3$ chronological splits, giving $63$ paired units that compare the no-topology baseline with the topology-aware guarded version of the same architecture. The audit reports mean relative RMSE reduction, bootstrap confidence intervals over paired units, improved/worsened/tied counts, paired effect size $d_z$, and paired sign-flip/randomization tests on the signed RMSE improvements. A single current chronological split is also reported descriptively to show dataset-level behavior under the same validation-only selection rule.

\section*{Results}

\subsection*{Evaluation protocol}
The Results section follows the no-leakage protocol defined above. For the deep architectures, the no-topology and topology-aware members of each pair were trained from matched initial weights with matched training seeds and data-loader shuffling. The final guarded prediction was selected by validation RMSE, so the topology-aware mechanism is evaluated as a validation-selected inductive bias rather than as an unconditional replacement for the base architecture.

\subsection*{Representative chronological split}
Table~\ref{tab:architecture_comparison} reports the representative current chronological split for all seven dataset units and three architecture families. On this split, the lightweight attention/Ridge model improved on six of seven datasets. \texttt{PatchTSTForRegression} improved on the higher-topology and shell datasets and retained the baseline on four additional datasets. \texttt{TimeSeriesTransformerForPrediction} improved on five datasets, tied on one, and worsened on the cyclic-loop benchmark.

\begin{table}[H]
\centering
\caption{Current matched chronological-split RMSE comparison. Guarded topology values include the baseline-level validation guard; equal values indicate that the baseline prediction was retained. Lower RMSE is better.}
\label{tab:architecture_comparison}
\scriptsize
\resizebox{\linewidth}{!}{%
\begin{tabular}{lcccccc}
\toprule
Dataset & \multicolumn{2}{c}{Lightweight/Ridge} & \multicolumn{2}{c}{PatchTST} & \multicolumn{2}{c}{TimeSeriesTransformer} \\
\cmidrule(lr){2-3}\cmidrule(lr){4-5}\cmidrule(lr){6-7}
 & Baseline & Guarded & Baseline & Guarded & Baseline & Guarded \\
\midrule
Higher topology (H$_2$ gap) & 0.3329 & 0.0344 & 0.4970 & 0.0383 & 0.5082 & 0.0383 \\
Cyclic topology (H$_1$ loop) & 0.5956 & 0.3602 & 0.0957 & 0.0957 & 0.1180 & 0.1498 \\
Shell topology (H$_2$) & 0.1306 & 0.1202 & 0.4985 & 0.1809 & 0.4228 & 0.1795 \\
CO$_2$ monthly & 0.8742 & 0.6350 & 0.8488 & 1.0258 & 11.2122 & 1.0256 \\
S\&P/FRED return shape & 0.1522 & 0.1148 & 0.0486 & 0.0486 & 0.0612 & 0.0612 \\
IMS bearing 1 & 0.0253 & 0.0255 & 0.0246 & 0.0246 & 0.0455 & 0.0355 \\
IMS bearing 2 & 0.6448 & 0.5083 & 0.5383 & 0.5383 & 1.1721 & 0.6951 \\
\bottomrule
\end{tabular}%
}
\end{table}

The representative split shows that the topology-aware mechanism is most effective when the target is tied to stable geometric structure. The higher-topology and shell datasets improved for all three architectures. The S\&P/FRED return-shape dataset was more conservative for the deep models, with the guard retaining the no-topology baseline for both transformer architectures. The CO$_2$ series was strongly improved by the TimeSeriesTransformer variant but not by PatchTST, indicating that topology-aware gains remain architecture-dependent even when the same train/validation/test protocol is used.

\subsection*{Repeated seed/split empirical effect-size audit}
The primary inferential analysis repeated the paired comparison across three random seeds and three chronological splits for each dataset, yielding $63$ paired units per architecture. Table~\ref{tab:significance_summary} reports paired empirical effects for the topology-aware guarded model versus the corresponding no-topology baseline. Positive relative RMSE reduction indicates lower RMSE for the topology-aware guarded model. The bootstrap interval summarizes uncertainty in the mean relative reduction, and the paired randomization test is applied to signed RMSE improvements.

\begin{table}[H]
\centering
\caption{Repeated matched seed/split empirical effect-size audit for topology-aware guarded models versus their no-topology baselines. Each architecture contributes 63 paired units: seven datasets, three seeds, and three chronological splits.}
\label{tab:significance_summary}
\scriptsize
\resizebox{\linewidth}{!}{%
\begin{tabular}{lcccccc}
\toprule
Architecture & Units & Improved / worsened / tied & Mean relative RMSE reduction & 95\% bootstrap CI & Paired effect size $d_z$ & Randomization $p$ \\
\midrule
Lightweight attention/Ridge & 63 & 46 / 17 / 0 & 12.5\% & [-5.3, 27.7]\% & 0.44 & $7.2\times10^{-4}$ \\
PatchTSTForRegression & 63 & 33 / 10 / 20 & 23.5\% & [14.1, 33.2]\% & 0.55 & $3.5\times10^{-5}$ \\
TimeSeriesTransformerForPrediction & 63 & 47 / 6 / 10 & 47.8\% & [38.4, 57.0]\% & 0.47 & $<10^{-4}$ \\
\bottomrule
\end{tabular}%
}
\end{table}

The repeated audit shows statistically significant signed paired effects for all three architectures, but the effects are not uniform across datasets. The lightweight attention/Ridge model improved in 46 of 63 paired units and had a mean relative RMSE reduction of 12.5\% with a two-sided paired randomization $p=7.2\times10^{-4}$. The 95\% bootstrap interval for relative reduction crosses zero, reflecting heterogeneous relative magnitudes across datasets, but the signed paired audit remains strongly positive. \texttt{PatchTSTForRegression} improved in 33 units, worsened in 10 units, and retained the baseline in 20 units; its mean relative reduction was 23.5\% with a 95\% bootstrap interval of [14.1, 33.2]\% and a paired randomization $p=3.5\times10^{-5}$. \texttt{TimeSeriesTransformerForPrediction} produced the largest mean relative reduction, 47.8\%, with a 95\% bootstrap interval of [38.4, 57.0]\% and a two-sided paired randomization probability below $10^{-4}$. This estimate should be interpreted together with the dataset-level results below because large reductions on strongly geometric tasks contribute substantially to the mean.

\subsection*{PatchTSTForRegression dataset-level behavior}
\texttt{PatchTSTForRegression} showed a significant aggregate effect, but the dataset-level pattern was heterogeneous. Table~\ref{tab:patchtst_dataset_results} shows that the strongest PatchTST gains occurred on the higher-order synthetic topology task and the shell topology task, where the guarded model improved in every repeated unit. The IMS bearing datasets showed smaller positive effects. In contrast, the S\&P/FRED return-shape, CO$_2$, and cyclic-loop tasks were weak or negative for PatchTST, and the validation guard frequently retained the no-topology baseline.

\begin{table}[H]
\centering
\caption{PatchTSTForRegression repeated seed/split dataset-level audit. Values are averaged over three seeds and three chronological splits for each dataset.}
\label{tab:patchtst_dataset_results}
\scriptsize
\resizebox{\linewidth}{!}{%
\begin{tabular}{lccccc}
\toprule
Dataset & Units & Improved / worsened / tied & Baseline RMSE & Guarded RMSE & Mean relative reduction \\
\midrule
Higher topology (H$_2$ gap) & 9 & 9 / 0 / 0 & 0.4958 & 0.0324 & 93.5\% \\
IMS bearing 1 & 9 & 6 / 1 / 2 & 0.0248 & 0.0208 & 12.4\% \\
IMS bearing 2 & 9 & 4 / 1 / 4 & 0.3837 & 0.3646 & 3.4\% \\
S\&P/FRED return shape & 9 & 1 / 1 / 7 & 0.0683 & 0.0686 & -0.7\% \\
CO$_2$ monthly & 9 & 2 / 4 / 3 & 1.0178 & 1.0614 & -7.6\% \\
Cyclic topology (H$_1$ loop) & 9 & 2 / 3 / 4 & 0.0897 & 0.0921 & -3.4\% \\
Shell topology (H$_2$) & 9 & 9 / 0 / 0 & 0.4998 & 0.1639 & 67.2\% \\
\bottomrule
\end{tabular}%
}
\end{table}

These results indicate that PatchTST is compatible with the topology-aware guard, but its benefit depends on how well patch-level temporal representations preserve the relevant geometry. Higher-order shape structure produced large and reliable improvements, whereas weakly geometric or low-amplitude return-shape tasks produced small or negative effects. The validation guard is therefore an essential part of the reported PatchTST behavior: in 20 of 63 PatchTST units, the final selected prediction was the no-topology baseline.

\subsection*{Geometry-specific behavior}
The repeated results indicate that the strongest gains occur when the predictive signal is aligned with the persistent homology features supplied to the attention mechanism. The higher-topology synthetic task and the shell topology task improved consistently for all architectures, with PatchTSTForRegression improving in all nine repeated units for both datasets. TimeSeriesTransformerForPrediction also showed robust gains on the CO$_2$ and IMS bearing tasks, whereas the S\&P/FRED return-shape task was mostly guarded back to the baseline for the deep models. This pattern is consistent with the interpretation that topology-aware attention is most useful when the topology summary captures stable geometry that is not already represented by the temporal encoder.

Figure~\ref{fig:sp500_diag} shows the S\&P/FRED topology diagnostic. In this dataset, the topology-aware local residual did not consistently improve the deep architectures, and the baseline-level guard often selected the no-topology prediction. The diagnostic illustrates the conservative behavior of the method: topology contributes when validation evidence supports it, while the baseline is retained when the geometry channel is not predictive.

\begin{figure}[H]
  \centering
  \includegraphics[width=0.95\linewidth]{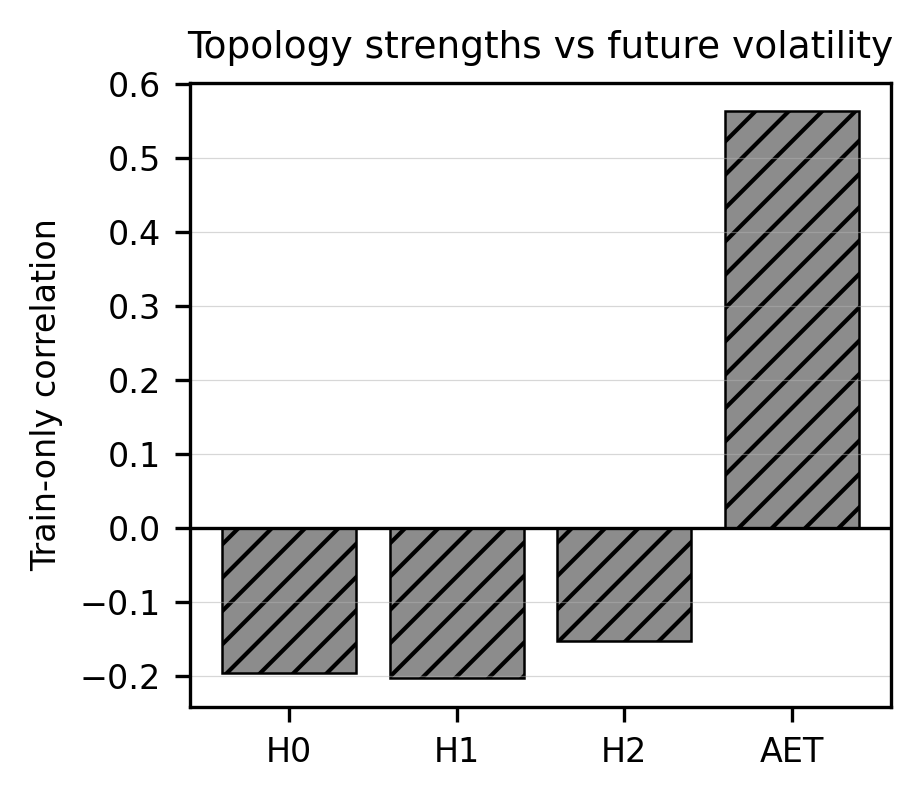}
  \caption{Diagnostic summary for the S\&P/FRED return-shape dataset. The current matched split shows that topology-aware predictions are retained only when they satisfy the validation guard; otherwise the baseline prediction is selected.}
  \label{fig:sp500_diag}
\end{figure}

Overall, the results support topology-aware guarded attention as an architecture-dependent inductive bias. The repeated audit provides evidence for positive signed paired effects in the lightweight attention/Ridge model, \texttt{PatchTSTForRegression}, and \texttt{TimeSeriesTransformerForPrediction}. The dataset-level results show that the benefit is concentrated in units with stable geometric structure and should be interpreted as a guarded, validation-selected improvement rather than as an unconditional replacement for standard temporal modeling.

\section*{Ablation and Interpretation}
The supplementary CSVs contain the completed lightweight attention/Ridge candidates: classical attention, Zeng-style local $H_0$, static topology, validation blends, dynamic $\eta$, learned-$\eta$ Ridge, RKHS-only modes, hybrid Euclidean/RKHS modes, local Hilbert/PH variants, and guarded residual rows. The Transformer architecture folders contain completed rows for \texttt{baseline\_no\_topology}, \texttt{topology\_global\_logit\_bias}, and \texttt{topology\_global\_plus\_guarded\_local\_hilbert\_ph}. The selected modes are data-dependent: the higher-topology stress test selects AET globally and local Hilbert/PH locally; CO$_2$ selects RKHS-full topology for the lightweight model; S\&P 500 selects AET in the completed lightweight run; IMS set 2 selects RKHS-full dynamic topology; and the cyclic benchmark selects a hybrid full dynamic blend. This supports the main methodological claim that topology should be chosen to match the geometry of the dataset rather than forced as a single fixed channel.

\begin{table}[H]
\centering
\small
\caption{Relative RMSE reduction (\%) for the completed lightweight attention/Ridge architecture. Positive values indicate improvement, i.e., lower RMSE.}
\label{tab:relative_reductions}
\resizebox{0.82\linewidth}{!}{%
\begin{tabular}{lccc}
\toprule
Dataset & \makecell{Local vs\\Zeng-style $H_0$} & \makecell{Global vs\\Zeng-style $H_0$} & \makecell{Global vs\\Classical} \\
\midrule
Higher-topology stress & 92.1 & 89.2 & 84.2 \\
Cyclic $H_1$ benchmark & 41.3 & 31.0 & 28.9 \\
Shell/void $H_2$ benchmark & -0.3 & 16.1 & 8.0 \\
Monthly CO$_2$ & 0.8 & 18.4 & 27.4 \\
S\&P 500 volatility & 0.0 & -30.5 & 24.6 \\
NASA IMS set 1 & 0.2 & 14.7 & -8.5 \\
NASA IMS set 2 & 0.0 & 31.5 & 47.4 \\
\midrule
Average (all rows) & 19.2 & 24.3 & 30.3 \\
\bottomrule
\end{tabular}%
}
\vspace{0.25em}
\parbox{0.82\linewidth}{\footnotesize The average is the unweighted mean across the seven non-degenerate comparisons; negative values are retained rather than selectively excluded.}
\end{table}

\section*{Discussion}
The results support a scale-sensitive view of topology in attention. Zeng-style local $H_0$ topological attention remains a strong simple baseline when the predictive signal is locally coherent. The proposed global mechanism is more expressive because it uses full-window $H_0$, $H_1$, $H_2$, AET, and RKHS topology, while the local Hilbert/PH residual preserves the local $H_0$ case through Proposition~1 and adds higher local topology when scalar connectivity is insufficient. The repeated paired audit shows that this design produces statistically significant empirical effects not only in the lightweight attention/Ridge model, but also in \texttt{PatchTSTForRegression} and \texttt{TimeSeriesTransformerForPrediction}.

The architecture-level results also show that topology should be treated as a validation-selected inductive bias. The strongest and most stable improvements occurred on the higher-topology and shell benchmarks, where the target construction directly depends on geometric structure. The S\&P/FRED return-shape task and selected PatchTST runs were more conservative, often retaining the no-topology baseline. This behavior is desirable for scientific forecasting: the method should expose geometric information when it is predictive, while preserving the ordinary temporal model when the topology channel does not improve validation performance.

The statistical evidence should therefore be read as evidence of a positive guarded effect across the repeated benchmark units, not as evidence that topology improves every dataset or every architecture instantiation. The bootstrap interval for the lightweight relative reduction crosses zero, and the larger transformer reductions are influenced by datasets where the target is explicitly geometric. This heterogeneity is consistent with the central premise of the method: topology is useful when the predictive mechanism is geometric and should be suppressed by validation when it is not.

The Hilbert/RKHS component is useful because the kernel bandwidth controls geometric scale. For small bandwidths, the Hilbert distance emphasizes local neighborhoods; for larger bandwidths, it behaves more globally. Persistent homology under the induced Hilbert distance can therefore reveal nonlinear similarity structure that Euclidean distances miss. However, RKHS topology does not by itself reproduce local window decompositions; the explicit local cover is required for the containment proof and for the validation-gated residual used in the experiments.

\section*{Limitations and Reproducibility}
Exact persistent homology can be expensive for long windows, so the implementation uses capped exact computations when available and smooth surrogates otherwise. The smooth $H_1$ and $H_2$ biases are interpretable approximations, not replacements for exact persistence in every setting; in particular, the $H_2$ term should be interpreted as a shell-sensitive proxy rather than as exact $H_2$ persistence. The dense $O(N^2)$ topology biases and the naive $O(N^3)$ smooth $H_1$ surrogate also limit direct scalability to very long sequences, and the present manuscript does not provide a systematic wall-clock benchmark against non-topological baselines. The repeated audit uses three seeds and three chronological splits across seven dataset units per architecture, which provides a stronger paired assessment than a single split but remains limited in dataset breadth. Future work should test broader public forecasting collections, additional industrial degradation datasets, larger transformer families such as Informer, Autoformer, and related long-horizon architectures, and optimized sparse or low-rank topology approximations. The controlled Zeng-style baseline is designed to isolate local $H_0$ topology fairly; it is not a full reproduction of every engineering detail in the original Zeng et al. implementation. The code and executed notebook in the package include seeds, hyperparameter grids, dataset generation routines, train-only calibration files, repeated split outputs, and all result CSVs.

Ethically, the method is intended for scientific forecasting and predictive maintenance. The main risk is over-interpreting topological gains when the dataset geometry or target construction is poorly matched. The no-leakage protocol and IMS target sanity checks are included to reduce this risk.

\section*{Conclusion}

This research presents a topology-aware attention framework for time-series forecasting that operates at both global and local scales. Unlike approaches that append topological descriptors only as ancillary features, the global module modifies the token-interaction graph by injecting persistent homology, anchored Euler-transform signatures, and kernel-Hilbert geometry directly into the attention logits. To accommodate localized coherent signals, the framework pairs this global mechanism with a validation-gated local Hilbert/PH residual that contains earlier local $H_0$ baselines as a special case while extending representational capacity to local $H_1$, $H_2$, and RKHS geometries.

The empirical results support this geometry-matched inductive bias while also showing that its value is architecture-dependent. In the repeated paired audit, topology-aware guarded configurations improved in 126 of 189 architecture--dataset--seed--split comparisons, worsened in 33, and retained the baseline in 30. The signed paired tests were positive for all three architecture families. Lightweight attention/Ridge achieved a mean relative RMSE reduction of 12.5\% across 63 paired units ($p=7.2\times10^{-4}$), \texttt{PatchTSTForRegression} achieved a 23.5\% reduction ($p=3.5\times10^{-5}$), and \texttt{TimeSeriesTransformerForPrediction} achieved a 47.8\% reduction ($p<10^{-4}$). The strongest gains occurred in domains governed by explicit geometric structure, especially the higher-topology stress test and shell/void benchmark, while weaker or low-amplitude geometric tasks more often selected the no-topology baseline.

These findings indicate that topology is best treated as a targeted inductive bias rather than a universally forced replacement for ordinary attention. Some datasets and architecture-specific runs, especially weakly geometric PatchTST and financial return-shape cases, favored the no-topology baseline. The validation-gated design is therefore central to the method: it preserves the ordinary architecture when local or global topological information is not supported by held-out data, while allowing higher-order geometry to contribute when connectivity, cycles, shells, directional Euler changes, or nonlinear metric neighborhoods are predictive.

Overall, the work establishes an architecture-compatible design principle. Because the framework acts through additive logit biases and gated residuals, it can be integrated into Transformer-style forecasting architectures without discarding standard dot-product attention. By mathematically preserving baseline capacity while exposing global and local topological structure, topology-aware attention equips predictive models to recognize geometric relationships that classical dot products are not designed to encode explicitly.

\appendix
\section*{Appendix A: Synthetic Data Generation Details}
For the higher-topology stress test, the coherent loop class is generated by Eq.~\eqref{eq:stress_loop}; the scrambled class is generated by Eq.~\eqref{eq:stress_scramble}. Both are randomly rotated, centered, standardized coordinate-wise, perturbed with small Gaussian noise, and token-permuted. The target is a noisy loop-health indicator. The construction intentionally preserves simple coordinate statistics while changing pairwise loop geometry.

For the cyclic $H_1$ benchmark, each window contains sine, cosine, and weak auxiliary coordinates with randomized amplitude, phase, and frequency. For the shell/void benchmark, shell samples use noisy unit-radius directions, while filled samples use random radial contraction. These datasets isolate whether $H_1$ and shell-like geometry can be useful in addition to scalar local $H_0$.

\section*{Appendix B: Additional Result Files}
The package contains the executed notebooks, selected predictions, RMSE bar plots, topology-strength plots, target diagnostics, and CSVs for every mode. The most important folders are:
\begin{itemize}[leftmargin=*]
    \item \texttt{all\_results/resuls topology aware/}: lightweight attention/Ridge aggregate CSVs, selected comparisons, diagnostics, and figures.
    \item \texttt{all\_results/results PatchTSTForRegression/}: PatchTST aggregate CSVs, per-dataset validation/test RMSE files, predictions, training histories, and topology-calibration files.
    \item \texttt{all\_results/results TimeSeriesTransformerForPrediction/}: TimeSeriesTransformer aggregate CSVs, per-dataset validation/test RMSE files, predictions, training histories, and topology-calibration files.
    \item \texttt{all\_results/results empirical effect size significance/}: the current-split and repeated seed/split effect-size summaries, paired units, cross-architecture tables, and interpretation files used for Table~\ref{tab:significance_summary}.
    \item the Section 6 local Hilbert/RKHS guarded-output folder retained for reproducibility.
    \item the executed notebooks with the complete implementation and outputs.
    \item \texttt{scripts/generate\_journal\_figures.py}, which regenerates the main journal-style PDF figures from the bundled CSV outputs.
\end{itemize}

\end{document}